%% file: main.tex
\documentclass[runningheads]{llncs}
\usepackage[T1]{fontenc}
\usepackage[utf8]{inputenc}
\usepackage[disable]{todonotes}
\usepackage{hyperref}
\usepackage{amssymb}
\usepackage{amsmath}

\title{A Modal Logic for Explaining some Graph Neural Networks}
\titlerunning{A Modal logic for explaining some GNNs}
\author{Pierre Nunn and Francois Schwarzentruber}
\institute{Université de Rennes}
\date{October 2022}

\usepackage[noend]{algpseudocode}

\input{macros}

\newcommand{\EXPTIME}{\mathsf{EXPTIME}}
\newcommand{\EXPTIMEoracleNP}{\mathsf{EXPTIME}^{\mathsf{NP}}}
\newcommand{\NP}{\mathsf{NP}}
\newcommand{\PSPACE}{\ensuremath{\mathsf{PSPACE}}}
\newcommand{\EXPSPACE}{\mathsf{EXPSPACE}}

\begin{document}

\maketitle

\begin{abstract}
In this paper, we propose a modal logic in which counting modalities appear in linear inequalities. We show that each formula can be transformed into an equivalent graph neural network (GNN). We also show that each GNN can be transformed into a formula. We show that the satisfiability problem is decidable. We also discuss some variants that are in PSPACE.
\end{abstract}

\section{Introduction}

Graph neural networks are used to learn a class of graphs or pointed graphs (a graph with a designated vertex). GNNs are used in many applications: social networks \cite{DBLP:journals/kbs/SalamatLJ21}, chemistry,  knowledge graphs etc. (see \cite{DBLP:journals/aiopen/ZhouCHZYLWLS20} for an overview of the applications of GNNs).
The Saint-Graal for explaning GNNs would be to provide an algorithm for the following problem:

\myproblem{Synthesis of an explanation}{a GNN $\aGNN$}{a logical formula $\phi$ such that $\semanticsof{\aGNN} = \semanticsof{\phi}$}

\noindent where $\semanticsof \aGNN$ is the class of pointed graphs recognized by the GNN $\aGNN$, and $\semanticsof \phi$ is the class of pointed graphs in which $\phi$ holds. In other words, the goal is to compute a formula $\phi$ (in modal logic, or in graded modal logic for example) that completely explains the class of graphs recognized by~$\aGNN$.

For instance, in a social network, a person is recommended by a GNN $A$ iff that person has at least one friend that is musician (the formula $\phi$ being for instance expressed in modal logic by $\Diamond musician$, where $\Diamond$ is the existential modal operator).
The synthesis of an explanation in some logic - let say first-order logic, modal logic, or graded modal logic - is highly challenging. In this paper, we tackle a less challenging problem but that goes in the same direction. We provide an algorithmic solution for tackling the following problems:

\begin{multicols}{2}
\myproblem{P1: Verification of an explanation}{a GNN $\aGNN$, a logical formula $\phi$}{yes, if $\semanticsof{\aGNN} = \semanticsof{\phi}$}

\myproblem{P2: Verification of an explanation}{a GNN $\aGNN$, a logical formula $\phi$}{yes, if $\semanticsof{\aGNN} \subseteq \semanticsof{\phi}$}

\myproblem{P3: Verification of an explanation}{a GNN $\aGNN$, a logical formula $\phi$}{yes, if $\semanticsof{\phi} \subseteq \semanticsof{\aGNN}$}

\myproblem{P4: Finding a counterexample}{a GNN $\aGNN$, a logical formula $\phi$}{yes, if $\semanticsof{\phi} \cap \semanticsof{\aGNN} \neq \emptyset$}
    
\end{multicols}

Here are kind of question instances that problems P1-4 are able to solve:
\begin{itemize}
    \item P1: is a recommended person a person that has at least one musician friend?
    \item P2: does any recommended person have a musician friend?
    \item P3: is any person that a musician friend recommended?
    \item P4: is it possible to recommend a person that has at least one musician friend? 
\end{itemize}

Our solution is a general methodology to solve the problems P1-4 which consists in representing everything ($\phi$ but also the GNN $A$) in logic. 
Interestingly, there is a neat correspondence between graded modal logic and GNNs (\cite{barcelo_logical_2020}, \cite{grohe_logic_2022}). Graded modal logic \cite{DBLP:journals/sLogica/Fattorosi-Barnaba85} is a modal logic offering the ability via the construction $\Diamond^{\geq k} \phi$ to say that a vertex has more that $k$ successors satisfying a formula $\phi$ We know that a GNN that is expressible in first-order logic (FO) is also captured by a formula in graded modal logic \cite{barcelo_logical_2020}. However the use of graded modal logic is problematic because we do not know how to represent \emph{any} GNN into it (in particular those who are not expressible in FO).

That is why we define a logic called $\logicKsharp$ which is more expressive enough to capture a reasonable class of GNNs while being able to express any formula in modal logic or graded modal logic. We then provide an algorithm for the satisfiability problem for $\logicKsharp$.
In this article, we capture AC-GNN (aggregation-combination graph neural networks) \cite{barcelo_logical_2020} that are defined by an aggregation function which is the sum of feature vectors, the combination functions being linear functions truncated with the activation function $max(0, min(1, x))$, and where the classification function is linear too. The $max(0, min(1, x))$ is called truncated reLU (see \cite{barcelo_logical_2020}) or clipped reLU (see \cite{DBLP:journals/symmetry/WangWL19}).

 The logic $\logicKsharp$  we consider is a combination of counting modalities and linear programming. It extends graded modal logic. We provide a translation from any GNN $A$ to a formula $tr(A)$ in $\logicKsharp$ so the problems P1-4 reformulate as follows:
\begin{itemize}
    \item P1: is $tr(A) \leftrightarrow \phi$ valid?
    \item P2: is $tr(A) \rightarrow \phi$ valid?
    \item P3: is $\phi \rightarrow tr(A)$ valid?
    \item P4: is $\phi \land tr(A)$ satisfiable?
\end{itemize}

The formula $\phi$ can be for instance a formula of modal logic K and graded modal logic. As $\logicKsharp$ subsumes these logics, all the problems P1-4 in fact reduce to the satisfiability problem of $\logicKsharp$ (recall that a formula valid if its negation is unsatisfiable). We prove that the satisfiability problem of $\logicKsharp$ is \emph{decidable}.

 Interestingly, given a formula, we are able to construct an equivalent GNN. This can be used to \emph{tune} an existing GNN.
Suppose you learnt a GNN $\aGNN$ but you aim at constructing a new GNN that behaves like $\aGNN$ but excludes the pointed graphs that do not satisfy $\phi$. More precisely, for the following problem:

\myproblem{Tuning of a GNN}{a GNN $\aGNN$, a logical formula $\phi$}{a new GNN $\aGNN'$ such that $\semanticsof{\aGNN'} = \semanticsof{\aGNN} \cap \semanticsof{\phi}$}

we simply take a GNN $\aGNN'$ that represents the formula $tr(\aGNN) \land \phi$. 
More precisely, the contributions of this paper are:
\begin{itemize}
    \item the formal definition of logic $\logicKsharp$;
    \item the construction of a GNN $\aGNN$ equivalent to a $\logicKsharp$-formula $\phi$ (it generalizes the result of Prop 4.1 in \cite{barcelo_logical_2020})
    \item the construction of a $\logicKsharp$-formula $tr(\aGNN)$ equivalent to $\aGNN$
    \item the fact that the satisfiability problem of $\logicKsharp$ is in $\EXPTIMEoracleNP$ (i.e. $\EXPTIME$ with an $\NP$ oracle).
    \item Restrictions of the satisfiability problem that are in \PSPACE.
\end{itemize}

\emph{Outline. } 
In Section~\ref{section:GNN} we recall the definition of AC-GNN. In Section~\ref{section:logic}, we define the logic $\logicKsharp$. In Section~\ref{section:correspondence}, we study the correspondence between GNN and logic $\logicKsharp$. In Section~\ref{section:decidability}, we discuss the satisfiability problem of $\logicKsharp$.

\section{Background on AC-GNN}
\label{section:GNN}

In this paper, we consider aggregate-combine GNN (AC-GNN) \cite{barcelo_logical_2020}, also sometimes called message passing neural
network (MPNN) \cite{grohe_logic_2022}. In the rest of the paper, we call a AC-GNN simply a GNN.

\newcommand{\dimensionstate}{d}
\newcommand{\nblayers}{L}
\newcommand{\layer}{\mathcal L}
\newcommand{\AGG}{AGG}
\newcommand{\COMB}{COMB}
\newcommand{\CLS}{CLS}

\begin{definition}[labeled graph]
A (labeled directed) graph $G$ is a tuple $(\setvertices, \setedges, \labeling)$ such that $\setvertices$ is a finite set of vertices, $\setedges \subseteq \setvertices \times \setvertices$ a set of directed edges and $\labeling$ is a mapping from~$\setvertices$ to a valuation over a set of atomic propositions. We write  $\ell(u)(p) = 1$ when atomic proposition $p$ is true in $u$, and $\ell(u)(p) = 0$ otherwise.
\end{definition}

\begin{definition}[state]
A state $x$ is a mapping from $\setvertices$ into $\setR^\dimensionstate$ for some $\dimensionstate$.
\end{definition}

As in \cite{scarselli_graph_2009}, we use the term `state' for both denoting $x$ and also the vector $x(v)$ at a given vertex $v$. 
Suppose that the relevant atomic propositions are $p_1, \dots, p_k$.
The initial state $x_0$ is defined by: $$x_0(u) = (\ell(u)(p_1), \dots, \ell(u)(p_k), 0, \dots, 0)$$ for all $u \in V$.
For simplicity, we suppose that all states are of the same dimension~$\dimensionstate$.


\begin{definition}[aggregation function/combination function]
An aggregation function $\AGG$ is a function mapping finite multisets of vectors in $\mathbb{R}^{\dimensionstate}$ to vectors in $\mathbb{R}^{\dimensionstate}$. A combination function $\COMB$ is a function mapping a vector in $\mathbb{R}^{2\dimensionstate}$ to vectors in $\mathbb{R}^{\dimensionstate}$.
\end{definition}


\begin{definition}[GNN layer]
A GNN layer of input/output dimension $p$ is defined by an aggregation function $\AGG$ and a combination function $\COMB$. 
\end{definition}

\begin{definition}[GNN]
A GNN is a tuple ($\layer^{(1)},...\layer^{(\nblayers)}, \CLS$) where $\layer^{(1)},...\layer^{(\nblayers)}$ are $d$ GNN layers and $\CLS : \setR^\dimensionstate \rightarrow \set{0, 1}$ is a classification function.
\end{definition}

When applied to a graph $G$, the $t$-th GNN layer $\layer^{(t)}$ transforms the previous state $\statet {t-1}$ into the next state  $\statet {t}$ by:
$$\statetv{t}u= \COMB (\statetv {t-1} u,\AGG(\multiset{\statetv {t-1} v | uv \in \setedges}))$$
\noindent
where $\AGG$ and $\COMB$ are respectively the aggregation and combination function of the $t$-th layer. 
In the above equation, note that the argument of $\AGG$ is the multiset of the state vectors of the successors of $v$. Thus, the same vector may occur several times in that multiset. Figure~\ref{figure:layer} explains how a layer works at each vertex.

\tikzstyle{vertex} = [circle,draw, inner sep=0mm,minimum height=3mm,font=\tiny]
\tikzstyle{processarrow} = [line width = 1mm, -latex]

\newcommand{\tikzexamplegraphm}[5]
{
\begin{tikzpicture}[xscale=1, yscale=0.8]
   \node[vertex, #2] (u) at (-1, 0) {};
   \node[vertex,  fill=blue!20!orange] (v) at (-0.5, 1) {#3};
   \node at (-0.5, 1.4) {$u$};
   \node[vertex,  #4] (w) at (0, 0) {};
   \node[vertex] (y) at (-1, -1) {#5};
   \node at (0, -1) {$#1$};
   \draw[->] (v) edge (u);
   \draw[->] (v) edge (w);
   \draw[->] (u) edge (y);
\end{tikzpicture}
}

\begin{figure}
    \centering
     \begin{tikzpicture}[node distance=30mm]
    \node (gs) at (-0.35, 0.2) {\tikzexamplegraphm{x_{t-1}}{fill=blue!20!white}{}{fill=blue!20!white}{}};
   \node[draw, right of=gs, blue] (AGG) {$\AGG$};
   \draw[processarrow, blue] (u) edge[bend right=20] (AGG);
   \draw[processarrow, blue] (w) edge[bend left=20] (AGG);
    \node[draw, right of=AGG] (COMB) {$\COMB$};
    \draw[processarrow, blue] (AGG) edge[bend left=0] (COMB);
    \draw[processarrow, orange] (v) edge[bend left=20] (COMB);
    \node[right of=COMB] (g1) {\tikzexamplegraphm{x_t}{}{}{}{}};
    \draw[processarrow, orange] (COMB) edge[bend left=20] (8.5, 1);
    \end{tikzpicture}
    \caption{A layer in a GNN transforms the state $x_{t-1}$ at step $t-1$ into the state $x_t$ at time $t$. The figure shows how $x_t(u)$ is computed. First, the function $\AGG$ is applied to the state in the successors of $u$. Then $\COMB$ is applied to that result and $x_{t-1}(u)$ to obtain $x_t(u)$.}
    \label{figure:layer}
\end{figure}
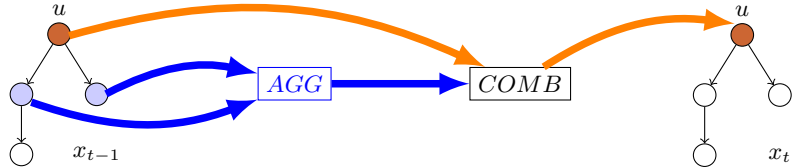

Figure~\ref{figure:gnn} explains how the overall GNN works: the state is updated at each layer; at the end the fonction $\CLS$ says whether each vertex is positive (1) or negative (0).

\newcommand{\tikzexamplegraph}[5]
{
\begin{tikzpicture}[scale=0.5]
\pgfmathparse{70*rnd+30}
\edef\tmp{\pgfmathresult}
   \node[vertex, fill=white!\tmp!black] (u) at (-1, 4) {#2};
   \pgfmathparse{70*rnd+30}
\edef\tmp{\pgfmathresult}
   \node[vertex,fill=white!\tmp!black] (v) at (-0.5, 5) {#3};
   \pgfmathparse{70*rnd+30}
\edef\tmp{\pgfmathresult}
   \node[vertex,fill=white!\tmp!black] (w) at (0, 4) {#4};
   \pgfmathparse{70*rnd+30}
\edef\tmp{\pgfmathresult}
   \node[vertex,fill=white!\tmp!black] (y) at (-1, 3) {#5};
   \node at (0, 3) {$#1$};
   \draw[->] (v) edge (u);
   \draw[->] (v) edge (w);
   \draw[->] (u) edge (y);
\end{tikzpicture}
}

\newcommand{\tikzexamplegraphoutput}[5]
{
\begin{tikzpicture}[scale=0.5]
   \node[vertex] (u) at (-1, 4) {#2};
   \node[vertex] (v) at (-0.5, 5) {#3};
   \node[vertex] (w) at (0, 4) {#4};
   \node[vertex] (y) at (-1, 3) {#5};
   \node at (0, 3) {$#1$};
   \draw[->] (v) edge (u);
   \draw[->] (v) edge (w);
   \draw[->] (u) edge (y);
\end{tikzpicture}
}

\newcommand{\tikzlayer}[1]{layer #1}

\begin{figure}
    \centering
    \begin{tikzpicture}[node distance=17mm]
    \node (g0) {\tikzexamplegraph{x_0}{}{}{}{}};
    \node[draw, right of=g0] (l1) {\tikzlayer{1}};
    \node[right of=l1]  (g1) {\tikzexamplegraph{x_1}{}{}{}{}};
     \node[draw, right of=g1] (l2) {\tikzlayer{2}};
    \node[right of=l2]  (g2) {\tikzexamplegraph{x_2}{}{}{}{}};
     \node[draw, right of=g2] (l3) {CLS};
    \node[right of=l3]  (g3) {\tikzexamplegraphoutput{}{1}{1}{0}{1}};
    \draw[processarrow] (g0) -- (l1);
    \draw[processarrow] (l1) -- (g1);
    \draw[processarrow] (g1) -- (l2);
    \draw[processarrow] (l2) -- (g2);
    \draw[processarrow] (g2) -- (l3);
    \draw[processarrow] (l3) -- (g3);
    \end{tikzpicture}
    \caption{General idea of a GNN with 2 layers applied on a graph with 4 vertices.}
    \label{figure:gnn}
\end{figure}
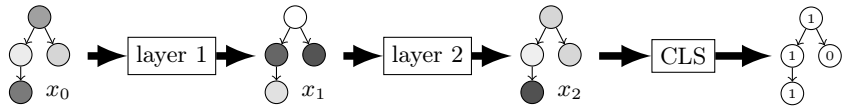

\begin{definition}
    Let $A$ be GNN. We define $\semanticsof A$ as the set of pointed graphs $(G, u)$ such that 
    $CLS(x_\nblayers(u)) = 1$.
\end{definition}

In the rest of the article, we suppose that the aggregation function is a sum:
$$\AGG (X) = \sum_{x \in X} x.$$


\section{Our proposal: logic $\logicKsharp$}
\label{section:logic}

In this section, we describe the syntax and semantics of $\logicKsharp$. We finish the section by defining its satisfiability problem.

\subsection{Syntax}

Consider a countable set $\Ap$ of propositions. We define the language of logic $\logicKsharp$ as the set of formulas generated by the following BNF:
\begin{align*}
    \phi & ::= p \mid \lnot \phi \mid \phi \lor \phi \mid E \geq 0 \\ 
    E & ::= c \mid \istrue\phi \mid \modalitynumber \phi \mid E + E \mid c\times E 
\end{align*}
where $p$ ranges over $\Ap$, and $c$ ranges over $\mathbb Z$. This logic is an extension of modal logic. Atomic formulas are propositions $p$, inequalities and equalities of linear expressions. We consider linear expressions over  $\istrue\phi$ and $\modalitynumber \phi$. The number $\istrue\phi$ is equal to 1 if $\phi$ holds in the current world and equal 0 otherwise. The number $\modalitynumber \phi$ is the number of successors in which $\phi$ hold. The language seems strict but we write $E_1 \leq E_2$ for $E_2 - E_1 \geq 0$, $E = 0$ for $(E \geq 0) \land (-E \geq 0)$, etc.

\begin{example}
Graded modal logic \cite{DBLP:journals/sLogica/Fattorosi-Barnaba85} extends classical modal logic by offering counting modality constructions of the form $\Diamond^{\geq k} \phi$ which means there are at least $k$ successors in which $\phi$ holds.
Logic $\logicKsharp$ is more expressive than Graded modal logic since $\Diamond^{\geq k} \phi$ is rewritten in  $k \leq \modalitynumber \phi$.    
\end{example}

\begin{example}
    Interestingly, the property `there are more $p$-successors than $q$-successors' can be expressed in logic $\logicKsharp$ by $\modalitynumber p \geq \modalitynumber q$, but cannot be expressed in FO, thus not graded modal logic. This is proven via a Ehrenfeucht-Fraïssé game.
\end{example}

\newcommand{\subformulasof}[1]{sub(#1)}
The set of subformulas, $\subformulasof{\phi}$ is defined by induction on $\phi$:
\begin{align*}
    \subformulasof p & = \set{p} \\
    \subformulasof {\lnot \phi} & = \set{\lnot \phi} \union \subformulasof{\phi} \\
    \subformulasof{\phi \lor \psi} & = \set{\phi \lor \psi} \union \subformulasof \phi \union \subformulasof \psi \\
    \subformulasof{E \geq 0} & = \set{E \geq 0} \union \bigcup \set{\subformulasof\psi \mid \text{$1_\psi$ or $\modalitynumber \psi$ appears in $E$}}
\end{align*}

\newcommand{\modaldepthof}[1]{md(#1)}
The modal depth of a formula, $\modaldepthof{\phi}$ and the modal depth of an expression, $\modaldepthof{E}$ are defined by mutual induction on $\phi$ and $E$:

	\begin{minipage}{3cm}
\begin{align*}
    \modaldepthof p & = 0 \\
    \modaldepthof {\lnot \phi} & = \modaldepthof{\phi} \\
    \modaldepthof {\phi \lor \psi} & = \max(\modaldepthof \phi,\modaldepthof \psi) \\
    \modaldepthof{E \geq 0} & = \modaldepthof{E} \\
\end{align*}
\end{minipage}
\begin{minipage}{3cm}
\begin{align*}
    \modaldepthof c & = 0 \\
    \modaldepthof {1_\phi} & = \modaldepthof{\phi} \\
    \modaldepthof {\modalitynumber \phi} & = \modaldepthof{\phi} + 1 \\
    \modaldepthof {E_1 + E_2} & = max(\modaldepthof {E_1},\modaldepthof {E_2}) \\
    \modaldepthof {k*E} & = \modaldepthof{E}
\end{align*}
\end{minipage}

~

As in modal logic, modalities are organized in levels.
\begin{example}
$\modaldepthof{(1_{p \land \modalitynumber {q} \leq 4} \leq \modalitynumber{(\modalitynumber p \geq 2}) \leq 4} = 2$.
The expressions $\modalitynumber {q}$ and $\modalitynumber{(\modalitynumber p \geq 2})$ are at the root level (level 1), while the expression $\modalitynumber p$ is at level~2.
\end{example}

A formula can be represented by a DAG (directed acyclic graph) instead of just a syntactic tree. It allows to share common expressions. For instance $\modalitynumber(p \land q) \geq 1_{p \land q}$ is represented by the following DAG in which $p \land q$ is used twice:
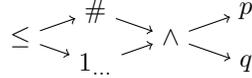
\begin{center}
\begin{tikzpicture}[yscale=0.7]
    \node (leq) at (0, 0) {$\leq$};
    \node (modalitynumber) at (1, 0.5) {$\modalitynumber$};
    \node (1) at (1, -0.5) {$1_{...}$};
    \node (land) at (2, 0) {$\land$};
    \node (p) at (3, 0.5) {$p$};
    \node (q) at (3, -0.5) {$q$};
    \draw[->] (leq) -- (modalitynumber);
    \draw[->] (leq) -- (1);
    \draw[->] (modalitynumber) -- (land);
    \draw[->] (1) -- (land);
    \draw[->] (land) -- (p);
    \draw[->] (land) -- (q);
\end{tikzpicture}    
\end{center}

\subsection{Semantics}

\newcommand{\semanticsvalue}[2]{[[#1]]_{#2}}

As in modal logic, a formula $\phi$ is evaluated in a pointed graph $(G, u)$ (also known as pointed Kripke model). 

\begin{definition}
We define the truth conditions $(G,u) \models \phi$ ($\phi$ is true in $u$) and the semantics $\semanticsvalue{E}{G,u}$ (the value of $E$ in $u$) of an expression $E$ by mutual induction on $\phi$ and $E$ as follows. 

\begin{center}
\begin{tabular}{lll}
$(G,u) \models p$ & if & $\labeling(u)(p) = 1$ \\
$(G,u) \models \neg \phi$ & if & it is not the case that $(G,u) \models \phi$ \\
$(G,u) \models \phi \land \psi$ & if & $(G,u) \models \phi$ and $(G,u) \models \psi$ \\
$(G,u) \models E \geq 0$ & if &  $\semanticsvalue{E}{G,u} \geq 0$ \\
\end{tabular}
\end{center}
\begin{center}
$\begin{array}{ll}
\semanticsvalue{c}{G, u} & = c \\
\semanticsvalue{E_1+E_2}{G, u} & = \semanticsvalue{E_1}{G,u}+\semanticsvalue{E_2}{G,u} \\
\semanticsvalue{c \times E}{G, u} & = c \times \semanticsvalue{E}{G,u} \\
\semanticsvalue{\istrue\phi}{G, u} & = \begin{cases}
1 & \text{if $(G,u) \models \phi$} \\
0 & \text{else}
\end{cases} \\  
\semanticsvalue{\modalitynumber\phi}{G, u} & = |\{v \in \setvertices \mid (u,v) \in E \text{ and } (G,v) \models \phi\}|
\end{array}$
\end{center}
\end{definition}

\begin{example}
    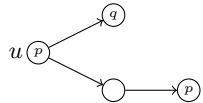
\begin{figure}
        \centering
       \begin{tikzpicture}[scale=1, rotate=90]
   \node[vertex] (u) at (-1, 0) {};
   \node[vertex] (v) at (-0.5, 1) {$p$};
   \node at (-0.5, 1.3) {$u$};
   \node[vertex] (w) at (0, 0) {$q$};
   \node[vertex] (y) at (-1, -1) {$p$};
   \draw[->] (v) edge (u);
   \draw[->] (v) edge (w);
   \draw[->] (u) edge (y);
\end{tikzpicture}
        \caption{Example of a pointed graph $G, u$. We indicate true propositional variables at each vertex.}
        \label{fig:pointedgraph}
    \end{figure}

    Consider the pointed graph $G, u$ shown in Figure~\ref{fig:pointedgraph}. We have $G, u \models p \land (\modalitynumber \lnot p \geq 2) \land \modalitynumber (\modalitynumber p \geq 1) \leq 1$. Indeed, $p$ holds in $u$, $u$ has (at least) two successors in which $\lnot p$ holds. Moreover, there is (at most) one successor which has at least one $p$-successor.
\end{example}


\begin{definition}
    $\semanticsof \phi$ is the set of the pointed graphs $G, u$ such that $G, u \models \phi$.
\end{definition}

\begin{definition} We say that $\phi$ is satisfiable when there exists a pointed graph $G, u$ such that $G, u \models \phi$.
    The satisfiability problem is: given $\phi$ in the language of $\logicKsharp$, is $\phi$ satisfiable?
\end{definition}



\section{Correspondence}
We explain how to transform a $\logicKsharp$-formula into a GNN,  and vice versa.

\subsection{From logic to GNN}
\label{section:correspondence}

\newcommand{\existsC}[1]{\exists^{\geq#1}}

Let us show that each $\logicKsharp$-formula is captured by a GNN. The proof follows the same line that the proof of the fact that each formula of graded modal logic is captured by a GNN (see Prop 4.1 in \cite{barcelo_logical_2020}). However, our first result (point 1 in the following theorem) is a generalisation of their result since logic $\logicKsharp$ is more expressive than graded modal logic. Moreover, point 2 of the following theorem explicitly mentions a bound on the number of layers in the GNN.

\begin{theorem}
For each
 $\logicKsharp$-formula $\phi$, 
 we can compute a GNN $A$ such that $\semanticsof \phi= \semanticsof A$. Furthermore, we have:
 \begin{enumerate}
    \item Either the number of layers and the dimension of the states in $A$ is $|\phi|$;
     \item Or the number of layers of $A$ is $O(md(\phi))$.
 \end{enumerate}
\end{theorem}

\begin{proof}
Let $\phi$ be a $\logicKsharp$ formula. Let $(\phi_1,...\phi_{L})$ be an enumeration of the sub-formulas of $\phi$ such that $\phi_L = \phi$.
We will construct a GNN $\mathcal{A}_{\phi}$ with $L$ layers. The dimension of the states is $L$. The goal is that for all $k \leq n$, the $k$-th component of $\statetv \ell v$ is equal to 1 if the formula $\phi_k$ is satisfied in node $v$, or 0 otherwise. 
The aggregation and combination function in each layer are set by:
\begin{align*}
    \AGG (X)   & = \sum_{x \in X} \textbf{x} \\
    \COMB(x,y) & = \sigmabold(xC+yA +b)
\end{align*}
where $\sigmabold$ is the function that applies componentwise the function\linebreak[4] $\sigma(x) = min((max(0,x),1)$ 
and where $A,C \in \setR^{L\times L}$ and $b \in \setR^{L}$ are defined as follows. All cells are zeros, except the cells given in the following table:

\begin{center}
\begin{tabular}{lll}
$\phi_\ell$ & $\ell$-th columns of $C$ and $A$ ~~~~~~~ & $b_\ell$ \\
\hline 
$p$ & $C_{\ell \ell} = 1$  & 0 \\
$\lnot \phi_i$ & $C_{i \ell} = -1$ & 1 \\
$\phi_i \lor \phi_j$ &  $C_{i \ell} = C_{j\ell} = 1$ & 0 \\
$\phi_i \land \phi_j$ & $C_{i \ell} = C_{j\ell} = 1$ & -1 \\
$c \leq \sum_{i \in I} k_i\times 1_{\phi_i} + \sum_{i \in I'} k_i \times \modalitynumber{\phi_i}$
~~~~~~ & 
$C_{i\ell} = k_i$ for $i \in I$ 
& $-c+1$ \\
& $A_{k\ell} = k_i$ 
for $i \in I'$ & 
\end{tabular}
\end{center}

For proving point 2., the idea is to transform each propositional level of $\phi$ into a CNF. We obtain a $\logicKsharp$-formula $\phi'$ potentially exponentially larger than~$\phi$. 


The advantage is now that each propositional level is a CNF and thus is of depth at most 2. We treat an arbitrary large disjunction or conjunction as a single step of computation. In other words, we now consider  an enumeration $(\phi_1,...\phi_{L})$ of subformulas of $\phi'$ such that $\phi_L = \phi'$ and with $L = O(md(\phi))$. Here are the corresponding $\ell$-columns of $C$ and $b_\ell$ for the cases where $\phi_\ell$ is an arbitrary large disjunction or conjunction:

\begin{center}
\begin{tabular}{lll}
$\phi_\ell$ & $\ell$-th column of $C$  & $b_\ell$ \\
\hline 
$\bigvee_{i \in I} \phi''_i \vee \bigvee_{i \in I'} \neg \phi''_i $ & $C_{i\ell} = \begin{cases}1 \text{ for $i \in I$} \\
-1 \text{for $i \in I'$} 
\end{cases}$
& $\card{I'}$ \\
$\bigwedge_{i \in I} \phi''_i $ & $C_{i\ell} = 1$ for $i \in I$  & $-\card{I}+1$
\end{tabular}
\end{center}

\end{proof}

\begin{example}
  Consider the formula $\phi = p \wedge (8 \leq 3\times \modalitynumber q) $.
    We define the following GNN $\aGNN$ which is equivalent to $\phi$ as follows. The aggregation function at each layer is
    $\AGG (X) = \sum_{x \in X} x$.
The combination function for each layer is
$\COMB(x,y) = \sigmabold(xC+yA+b) $
where 
$C = \begin{pmatrix}
1 & 0 & 0 & 1\\
0 & 1 & 0 & 0\\
0 & 0 & 0 & 1\\
0 & 0 & 0 & 0
\end{pmatrix}$,
$A = \begin{pmatrix}
0 & 0 & 0 & 0 \\
0 & 0 & 3 & 0\\
0 & 0 & 0 & 0\\
0 & 0 & 0 & 0\\
\end{pmatrix}$, and 
$b = \begin{pmatrix}
0 & 0 & -7 & -1\\
\end{pmatrix}$.
The columns in the matrices (from top to bottom) are respectively evaluated the following subformulas in that order: $p$, $q$, $8 \leq \modalitynumber q$, $\phi$. 

\end{example}

\subsection{From GNN to logic}

In this subsection, we show how to compute a $\logicKsharp$-formula that is equivalent to a GNN. Note that this direction was already tackled for graded modal logic for the subclass of GNNs that are FO-expressible, but their proof is not constructive~\cite{barcelo_logical_2020}.

\begin{theorem}
Let $A$ be a GNN $A$ with all aggregation function being $\AGG (X) = \sum_{x \in X} x$, and the classfication function being linear: $CLS(x) = \sum_i a_i x_i \geq 0$. Then we can compute  in poly-time in $|A|$ a $\logicKsharp$-formula $tr(A)$ represented as DAG such that $\semanticsof{A} = \semanticsof{tr(A)}$.
\end{theorem}


\begin{proof}
    Let us consider a GNN $A$ of $L$ layers where the aggregation function is always:
    $\AGG (X) = \sum_{x \in X} \textbf{x}.$

    The idea is that we represent the state $x_0(v)$ at all vertices $v$ by the truth value of some formulas. Initially, the states is represented by the formulas\linebreak[4] $(p_1, \dots, p_k, \bot, \dots, \bot)$.

    Suppose that states $x_t(v)$ are represented by the formulas $(\phi_1, \dots, \phi_\dimensionstate)$.
Then if the combination function is
    $\COMB(x,y) = \sigmabold(xC+yA +b) $
    then the states $x_{t+1}(v)$ are represented by the formulas $(\phi'_1, \dots, \phi'_\dimensionstate)$ where~$\phi'_\ell$~is
    $$\sum_{i=1..d} 1_{\phi_i} C_{i\ell} + \sum_{i=1..d} \modalitynumber{\phi_i} A_{i\ell} + b_\ell \geq 1$$

Now, we have formulas $(\phi_1, \dots, \phi_\dimensionstate)$ to represent $x_L(v)$. As $\CLS$ is linear, the final formula is $1_{\CLS(1_{\phi_1}, \dots, 1_{\phi_\dimensionstate})}$.
    $CLS(x_\nblayers(u)) = 1$.

    
\end{proof}

\section{Decidability}
\label{section:decidability}

\todo{est-ce qu'on peut faire plus efficace que ce qui est fait dans cette section ?}

Let us give an algorithm to solve the satisfiability of $\logicKsharp$, inspired by classical tableau method for modal logic K \cite{gore1999tableau}, but taking linear constraints into account.

\subsection{Design of the algorithm}

First constructions $1_\psi$ are easy to treat. Either $\psi$ holds and we say that $1_\psi = 1$, or $\psi$ does not hold and we say that $1_\psi = 0$. However, treating naively $\modalitynumber \psi$ as variables in a linear program will unfortunately not work. Let us note $\modalitynumber \psi_1, \dots \modalitynumber \psi_n$ the variables of the form $\modalitynumber \psi$ that appear in $\phi$ and that are not in the scope of a $\modalitynumber$-modality. First, some $\psi_i$ may be unsatisfiable, thus $\modalitynumber {\psi_i} = 0$. But the issue is more subtle. For instance, we always have 
	$ \modalitynumber p + \modalitynumber \lnot p = \modalitynumber q + \modalitynumber \lnot q$ (1).
%
The reader may imagine even more involved interactions between the $\modalitynumber \psi_i$ than equation~(1).
To take this interactions into account, we consider all possible conjunctions of $\psi_i$ and $\lnot \psi_i$. We define for all words $w \in \set{0, 1}^n$:
$$conj_w := \bigwedge_{i = 1..n \mid w_i = 1} \psi_i \land \bigwedge_{i = 1..n \mid w_i = 0} \lnot \psi_i.$$

\begin{example}
    $conj_{0100} = \lnot \psi_1 \land \psi_2 \land \lnot \psi_3 \land \lnot \psi_4$.
\end{example}

We then introduce a variable in the linear program for each word $w$ that counts the number of successors in which $conj_w$ holds. We have $\modalitynumber \psi_i = \sum_{w \mid w_i = 1} x_w$.

\begin{example}
    How do we guarantee that $\modalitynumber p + \modalitynumber \lnot p = \modalitynumber q + \modalitynumber \lnot q$? Suppose that $\psi_1 = p, \psi_2 = \lnot p, \psi_3 = q, \psi_4 = \lnot q$.
    As $p \land \lnot p$ is unsatisfiable, we have $x_{1100} = x_{1101} = x_{1110} = x_{1111} = 0$. We write $x_{11**} = 0$. In the same way, as and $p \land \lnot p$ and $q \land \lnot q$ are unsatisfiable $x_{00**} = 0$, $x_{**00} = 0$ and $x_{**11} = 0$. Finally:
    \begin{align*}
        \modalitynumber p = x_{1010}+x_{1001} &&
      \modalitynumber \lnot p = x_{0110}+x_{0101} \\
        \modalitynumber q = x_{1010}+x_{0110} &&
      \modalitynumber \lnot q = x_{1001}+x_{0101}
    \end{align*}
    We see that $\modalitynumber p + \modalitynumber \lnot p = x_{1010}+x_{1001} + x_{0110}+x_{0101} = \modalitynumber q + \modalitynumber \lnot q$.
\end{example}

\begin{figure}[t]
\begin{algorithmic}[1]
\Function{sat}{$\phi$}
    \For{$S \gets hintikkaSet(\phi)$}
    	\State $S := $ the set of inequalities in $S$ (clean-up)
        \State Let $\modalitynumber \psi_1, \dots \modalitynumber \psi_n$ be the variables of the form $\modalitynumber \psi$ appearing in $S$
        \State Consider LP-variables $x_w$ for all $w \in \set{0, 1}^n$
        \State Replace $\modalitynumber \psi_i$ by $\sum_{w \mid w_i = 1} x_w$ in $S$
        \For{$w \in \set{0, 1}^n$}
        \If{not sat($conj_w$)}
        \State add constraint $x_w = 0$ to $S$
        \EndIf
        \EndFor
        
        \If{the oracle says that $S$ is ILP-satisfiable}
        \State \Return{true}
        \EndIf
    \EndFor
    \State \Return{false}
\EndFunction
\end{algorithmic}
\caption{Algorithm for checking the satisfiability of a $\logicKsharp$-formula $\phi$.}\label{figure:algorithm}
\end{figure}

 Instead of providing tableau rules, we decided to present a more abstract version with Hintikka sets (see Def. 6.24 in \cite{DBLP:books/cu/BlackburnRV01}). 
They can be thought as a possible way to completely apply Boolean rules while keeping consistent. We adapt the definition to our setting.

\begin{definition}
    [Hintikka set]
    A Hintikka set $\Sigma$ for formula $\phi$ is a smallest (for inclusion) set of subformulas such that:
    \begin{enumerate}
        \item $\phi \in\Sigma$;
        \item if $\psi_1 \land \psi_2 \in \Sigma$ then $\psi_1 \in \Sigma$ and $\psi_2 \in \Sigma$
        \item if $\psi_1 \lor \psi_2 \in \Sigma$ then $\psi_1 \in \Sigma$ or $\psi_2 \in \Sigma$
        \item for all $\psi$, $\psi \not \in \Sigma$ or $\lnot \psi \not \in \Sigma$
        \item $1_\psi$ appears in $\phi$, either $\psi \in \Sigma$ and $1_\psi = 1 \in \Sigma$, or $\lnot \psi \in \Sigma$ and $1_\psi = 0 \in \Sigma$.
    \end{enumerate}
\end{definition}

Point 1 says that $\phi$ should be true. In point 3, if $\psi_1 \lor \psi_2$ then one of the formula -- $\psi_1$ or $\psi_2$ -- should be true, without telling which one. Point 4 is the consistency. Point 5 makes the link between the truth of $\psi$ and the value of $1_\psi$.

\begin{example}
    Consider formula $\phi = p \land (\modalitynumber r \geq 1_{q})$. There are two possible Hintikka sets for formula $\phi$:
   $\set{p, \modalitynumber r \geq 1_{q}, 1_{q} = 1, q }$
   and  $\set{p, \modalitynumber r \geq 1_{q}, 1_{q} = 0, \lnot q }$.
\end{example}

The algorithm (see Figure~\ref{figure:algorithm}) consists in examining all possible Hintikka sets. For each of them, we extract the linear program (line 3). We then compute the integer linear program by considering the variables $x_w$ discussed above (line 6). Line 7: we call recursively the function sat on $conj_w$ and we add the constraint $x_w = 0$ in case $conj_w$ is unsatisfiable.

\subsection{Soundness and completeness}

\begin{proposition}
	\label{proposition:soundnessandcompleteness}
    $\phi$ is $\logicKsharp$-satisfiable iff $sat(\phi)$ returns true.
\end{proposition}

\begin{proof}
We prove the proposition by induction on $md(\phi)$. The induction works because $md(conj_w) < md(\phi)$.
\fbox{$\Rightarrow$}
Suppose that $\phi$ is satisfiable: let $G, v$ such that $G, u \models \phi$. Let us prove that $sat(\phi)$ returns true. We consider the Hintikka set $H$ made up of formulas that are true in $G, v$. The obtained $S$ is ILP-satisfiable because these equations and inequations are satisfied in $G, u$. Indeed, here is a solution: we set $x_w$ to be the number of successors of $u$ in which $conj_w$ hold. By induction, the call $sat(conj_w)$ are all correct: so if $sat(conj_w)$ returns false, then $conj_w$ is unsatisfiable. Thus there are no $u$-successors satisfying $conj_w$, and the constraints $x_w = 0$ (added line 8) hold. The number of $v$-successors satisfying $\psi_i$ is $\sum_{w \mid w_i = 1} x_w$.  So $S$ is ILP-satisfiable, and the algorithm returns true.

\fbox{$\Leftarrow$}
   Conversely, suppose that $sat(\phi)$ returns true. First, consider $S$ the corresponding Hintikka set for which the algorithm returned true. From $S$ we extract a valuation $\labeling(u)$ for the propositions to be set to true or false at a node $u$. By induction, for all $w \in \set{0, 1}^n$, the call $sat(conj_w)$ are all correct. Thus, if $conj_w$ is unsatisfiable, $x_w = 0$; otherwise $x_w$ is not constrained. As $sat(\phi)$ returned true, we know that $S$ is ILP-satisfiable. Consider a solution. If $x_w > 0$, we know that $conj_w$ is satisfiable. We consider $x_w$ copies of a pointed graph $G_w, u_w$ satisfying $conj_w$. We construct a model $G, u$ for $\phi$ as follows. We take $u$ as the point with valuation $\labeling(v)$. We then link $u$ to each point of the copies of $u_w$. The inequations in $S$ are satisfied at $G, u$. Indeed, the $u$-successors satisfying $\psi_i$ are exactly the $u_w$ with $w_i = 1$, and $\modalitynumber \psi_i$ is $\sum_{w \mid w_i = 1} x_w$. Thus, the obtained pointed $G, u$ is a model of $\phi$. Figure~\ref{figure:constructionmodel} shows an example of the construction.
\end{proof}

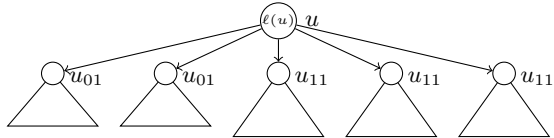
\begin{figure}
	\vspace{-5mm}
	\begin{center}
		\newcommand{\minimodel}[3]{
	 \begin{scope}[xshift=#1cm, yshift=-1cm]
	 		\draw (0,0) -- (0.4, -#3) -- (-0.4, -#3) -- cycle;
	 		\node[fill=white,vertex] (v) {};
	 		\draw[->] (u) -- (v);
	 		\node at (0.3, -0.1) {$u_{#2}$};
	 	\end{scope}	
	}
		\begin{tikzpicture}[xscale=1.5, yscale=0.7]
			\node[vertex] (u) {$\labeling(u)$};
			\node at (0.3, 0) {$u$};
		\minimodel {-2} {01} 1
		\minimodel {-1} {01} 1
		\minimodel {0} {11} {1.2}
		\minimodel {1} {11} {1.2}
		\minimodel {2} {11} {1.2}
		\end{tikzpicture}
	\end{center}
\vspace{-5mm}
	\caption{Example of a model $G, u$ \label{figure:constructionmodel} with $x_{00} = 0$, $x_{01} = 2$, $x_{10} = 0$ and $x_{11} = 3$.}
\end{figure}

Note that $\logicKsharp$ has the tree-model property as modal logic K. It can be proven by induction on $md(\phi)$, relying on the construction in the \fbox{$\Leftarrow$}-direction in the proof above.

\subsection{Complexity}

The recursive depth of our algorithm (Figure~\ref{figure:algorithm}) is bounded by the modal depth $md(\phi)$ of the initial formula $\phi$: the recursive tree is of depth  $md(\phi)$. Its branching factor is exponential in $|\phi|$. The number of nodes remains exponential in $|\phi|$. At each node, there is an exponential number of steps, provided that the ILP-solver is considered as an NP-oracle since integer linear programming (ILP) is in NP \cite{DBLP:journals/jacm/Papadimitriou81} (note that the linear programs computed here are of exponential size in $|\phi|$).
Our algorithm runs in exponential time in $|\phi|$, by calling a NP-oracle:  deciding the satisfiability problem of $\logicKsharp$ is in $\EXPTIMEoracleNP$. The class $\EXPTIMEoracleNP$ is defined as the class of decision problems decided by an algorithm running in exponential time (i.e. $2^{poly(n)}$) with a NP oracle, typically a SAT oracle or a ILP oracle (note that exponentially long ILP instances may be solved in one step) (see \cite{DBLP:series/lncs/Williams19}, \cite{DBLP:conf/coco/Hirahara15}).  Note that the complexity class  $\EXPTIMEoracleNP$ is included in the exponential hierarchy which is included in $\EXPSPACE$.

\begin{theorem}
    The satisfiability problem of $\logicKsharp$ is decidable and is in $\EXPTIMEoracleNP$.
\end{theorem}

\subsection{$\PSPACE$ subcases}

Let us discuss three types of restrictions to get \PSPACE-membership.

\paragraph{Bounding the number of $conj_w$. }
If we can limit the number of considered conjonctions $conj_w$, we may obtain a procedure running in polynomial space, making the restricted version of the satisfiability problem of $\logicKsharp$ in \PSPACE. For instance, if we know in advance that at each level of modal depth formulas in the scope of a $\modalitynumber$-modality are mutually unsatisfiable, then we do not need to consider all the conjunctions $conj_w$: all $conj_w$ are 0 except $conj_{10...0}$, $conj_{010...0}$, ..., $conj_{0...01}$. We keep a linear program polynomial in $|\phi|$.

\paragraph{Bounded the number of modalities. }
If we artificially make the syntactic restriction where we bound the number $n$ of $\modalitynumber \phi$ at each level, the satisfiability problem is also in \PSPACE. Indeed, $n$ becomes a constant, thus $2^n$ is a constant too. The size of $S$ is only polynomial in the size of $\phi$. 

\paragraph{Bounded branching. } Many graphs have bounded branching: grid graphs (of degree 4), sparse networks, etc.
If we ask whether a formula is satisfiable in a graph whose degree is bounded by a polynomial in the size of the input, then only a polynomial number of variables $x_w$ will be non-zero. The algorithm is then adapted by guessing the polynomial-size subset of variables $x_w$ that are non-zero. Again, we can run the algorithm in polynomial space.

\begin{proposition}
Let $k > 0$ be an integer.
    The satisfiability problem of $\logicKsharp$ restricted to graphs of degree at most $k$ is \PSPACE-complete.
\end{proposition}

\begin{proof}
    \PSPACE-membership comes from the discussion above. \PSPACE-hardness holds because the modal logic on graphs with at most 2 successors per worlds is \PSPACE-hard.  Write $\Box \psi$ as $\modalitynumber \lnot \psi \leq 0$. 
\end{proof}

Interestingly, there are fragments in which if a formula $\phi$ is satisfiable then~$\phi$ is satisfiable in a model of polynomial degree in $|\phi|$. Consider the fragment in which inequalities in formulas are of the form $\modalitynumber \psi \leq \modalitynumber \psi'$ (i.e. no addition, no multiplication by a scalar). Then if there is a solution, we can at each level have an ordering $\modalitynumber \psi_1 \leq \dots \leq \modalitynumber \psi_n$ where some of the $\leq$ may be strict. But then we can suppose that w.l.o.g. $0 \leq \modalitynumber \psi_{i+1} - \modalitynumber \psi_i \leq 1$. It means that the number of successors of each $\psi_i$ is $O(i)$; the number of successors is $O(n^2)$. We get:

\begin{proposition}
    The satisfiability problem of $\logicKsharp$ is $\PSPACE$-complete when inequalities in formulas are of the form $\modalitynumber \psi \leq \modalitynumber \psi'$.
\end{proposition}

\begin{proof}
    \PSPACE-membership comes from the discussion above. \PSPACE-hardness comes from the fact modal logic K is reducible to it. Write $\Box \psi$ as $\modalitynumber \top \leq \modalitynumber \psi$.
\end{proof}

\todo{à relire et vérifier}

\section{Related work}

Many works combine modal logic and quantitative aspects: counting (\cite{DBLP:conf/wollic/ArecesHD10}, \cite{DBLP:conf/aiml/Hampson16}), probabilities \cite{DBLP:conf/aaai/ShiraziA07}.
Linear programming and modal logic have already been combined to solve the satisfiability problem of graded/probabilistic modal logic \cite{DBLP:conf/lpar/SnellPW12}. Our logic $\logicKsharp$ can be seen as a `recursification' of the logic used in  
\cite{DBLP:journals/corr/abs-2206-05070}. They allow for counting successors satisfying a given feature, and not any subformula. Interestingly, they allow for counting also among all nodes in the graph (sort of counting universal modality). Their logic is proven to be undecidable by reduction from the Post correspondence problem. Contrary to our setting, they use their logic only to characterize labelled graphs, but not to give a back and forth comparison with the GNN machinery itself.

Modal logic has also been combined with neural network in the so-called \emph{
Connectionist modal logic} \cite{DBLP:journals/tcs/GarcezLG07} but it has no direct connection with GNNs.

Another solution would be to use directly explainable GNN such as those in \cite{DBLP:journals/corr/abs-2205-13234}. This is of course a deep debate: using models easy to use for learning, versus interpretable models \cite{DBLP:journals/natmi/Rudin19}.  The choice depends on the target application.

Yuan et al. \cite{DBLP:journals/pami/YuanYGJ23} provide a survey on methods used to provide explanations for GNNs by using black-box techniques.
 According to them, they are instance-level and model-level explanations. Instance-level explanations explain on why a graph has been recognized by an GNN; model-level ones how a given GNN works. For instance, they are also many methods based on Logic Explained Networks and variants to generate logical explanation candidates $\phi$ \cite{DBLP:journals/corr/abs-2210-07147}. Once a candidate is generated we could imagine use our problem P1 (given in the introduction) to check whether $\semanticsof A = \semanticsof \phi$, and thus being able to fully synthesize a trustworthy explanation. 
Our paper is clearly close to model-level explanations.

\todo{dire quelque chose d'intéressant sur $EXPTIME^NP$ \cite{DBLP:conf/coco/Hirahara15}}

\section{Perspectives}

We aim at considering a larger class of GNNs. This will need to augment the expressivity of the logic, for instance by adding reLU in the language. Fortunately SMT solvers have been extended to capture reLU \cite{DBLP:conf/cav/KatzBDJK17}.

Another possible direction would be to consider other classes of graphs. For instance, reflexive, transitive graphs. Restricted types of graphs lead to different modal logics: $KT$ (validities on reflexive Kripke models), $KD$ (on serial models), $S4$ (reflexive and transitive models), $KB$ (models where relations are symmetric), $S5$ (models where relations are equivalence relations), etc. \cite{DBLP:books/cu/BlackburnRV01} The logic $\logicKsharp$ defined in this paper is the counterpart of modal logic K with linear programs. In the same way, we could define $KT^\#$, $S4^\#$, $S5^\#$, etc. For instance, $KB^\#$ would be the set of validities of $\logicKsharp$-formulas over symmetric models;  $KB^\#$ would be the logic used when GNNs are only used to recognize undirected pointed graphs (for instance persons in a social network where friendship is undirected). In the future work, some connections between GNNs and logics designed to express properties over persons in social network, such as \cite{DBLP:conf/tark/SeligmanLG13} could be investigated.

A next direction of research would be to build a tool. The main difficulty is the complexity of the algorithm. However, we may rely on heuristics to guide the search (namely SAT solvers for computing only the relevant Hintikka sets, and relaxed linear programs). We could also directly use SMT solvers.

Of course, our Saint-Graal is the synthesis of a formula that matches a specification. This problem is close to the formula synthesis problem presented in \cite{DBLP:conf/aaai/PinchinatRS22}. An idea would be to represent the set of possible suitable explanation formulas by a grammar $G$ (for instance, the grammar restricted to graded modal logic) and to compute a formula generated by $G$ which is equivalent to $tr(A)$.




\bibliography{biblio}
\bibliographystyle{plain}

\appendix

\section{Appendix}

\subsection{Proof that $\logicKsharp$ is more expressive than FO (Example 2)}

We show that $\logicKsharp$ is more expressive than FO by proving that the formula $\modalitynumber p \geq \modalitynumber q$ is not expressible by a FO formula $\phi(x)$. We observe that if the property `for all vertices of a graph $\modalitynumber p \geq \modalitynumber q$' is not expressible in FO, then the FO formula $\phi(x)$ doesn't exist, because if it existed the property would be expressible in FO by the formula $\forall x \phi(x)$.

For each integer $n > 0$, we consider the graphs $A_n$ and $B_n$ such that every vertices of $A_n$ verify $\modalitynumber p \geq \modalitynumber q$ while this is not the case for $B_n$ :

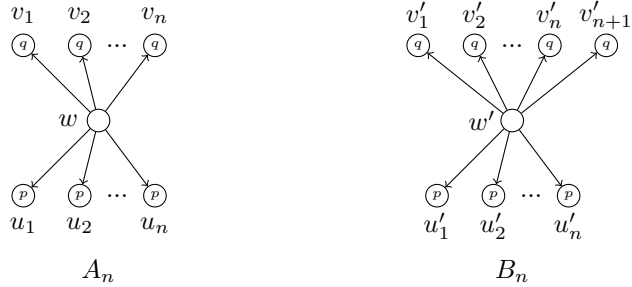
\begin{center}
\begin{tikzpicture}
   \node[vertex] (w) at (-0.5, 1) {};
   \node[vertex] (u1) at (-1.5, 0) {$p$};
    \node[vertex] (u2) at (-0.75, 0) {$p$};
    \node at (-0.25, 0) {$...$};
   \node at (-0.9, 1) {$w$};
    \node at (-1.5, -0.4) {$u_1$};
    \node at (-0.75, -0.4) {$u_2$};
    \node at (0.25, -0.4) {$u_n$};
   \node[vertex] (un) at (0.25, 0) {$p$};
    \node[vertex] (v1) at (-1.5, 2) {$q$};
    \node[vertex] (v2) at (-0.75, 2) {$q$};
    \node at (-0.25, 2) {$...$};
    \node at (-1.5, 2.4) {$v_1$};
    \node at (-0.75, 2.4) {$v_2$};
    \node at (0.25, 2.4) {$v_n$};
   \node[vertex] (vn) at (0.25, 2) {$q$};
   \draw[->] (w) edge (v1);
   \draw[->] (w) edge (v2);
    \draw[->] (w) edge (vn);
    \draw[->] (w) edge (u1);
    \draw[->] (w) edge (u2);
   \draw[->] (w) edge (un);

      \node[vertex] (bw) at (5, 1) {};
   \node[vertex] (bu1) at (4, 0) {$p$};
    \node[vertex] (bu2) at (4.75, 0) {$p$};
    \node at (5.25, 0) {$...$};
   \node at (4.6, 1) {$w'$};
    \node at (4, -0.4) {$u'_1$};
    \node at (4.75, -0.4) {$u'_2$};
    \node at (5.75, -0.4) {$u'_n$};
   \node[vertex] (bun) at (5.75, 0) {$p$};
    \node[vertex] (bv1) at (3.75, 2) {$q$};
    \node[vertex] (bv2) at (4.5, 2) {$q$};
    \node at (5, 2) {$...$};
    \node at (3.75, 2.4) {$v'_1$};
    \node at (4.5, 2.4) {$v'_2$};
    \node at (5.5, 2.4) {$v'_n$};
    \node at (6.25, 2.4) {$v'_{n+1}$};
   \node[vertex] (bvn) at (5.5, 2) {$q$};
   \node[vertex] (bvn1) at (6.25, 2) {$q$};
   \draw[->] (bw) edge (bv1);
   \draw[->] (bw) edge (bv2);
    \draw[->] (bw) edge (bvn);
    \draw[->] (bw) edge (bvn1);
    \draw[->] (bw) edge (bu1);
    \draw[->] (bw) edge (bu2);
   \draw[->] (bw) edge (bun);

  \node at (-0.5, -1) {$A_n$};
   \node at (5, -1) {$B_n$};
\end{tikzpicture}
\end{center}

\todo{citer un textbook+expliquer pour le pointed graph}

An $n$-round Ehrenfeucht-Fraïssé game is a game between two players, the spoiler and the duplicator played on two graphs $A = (V_A,E_A,\ell_A)$ and $B = (V_B,E_B,\ell_B)$. We suppose that  On each round the spoiler picks one graph and a vertex in this graph. The duplicator then chooses a vertex on the other graph. We have $n$ vertices ($a_1$,$a_2$,...,$a_n$) chosen in $A$ and $n$ vertices ($b_1$,$b_2$,...,$b_n$) chosen in $B$. The duplicator wins if and only if for $1 \leq i,j \leq n$ :

$a_i=a_j \Leftrightarrow b_i = b_j$

$(a_i,a_j) \in E_A \Leftrightarrow (b_i,b_j) \in E_B$

$\ell(a_i)= \ell(b_i)$

On the graphs $A_n$ and $B_n$, the duplicator wins the Ehrenfeucht-Fraïssé game with $n$ rounds: if the spoiler chooses $w$ (resp. $w'$) the duplicator chooses $w'$ (resp. $w$), if the spoiler chooses some $u_i$ or $v_i$ (resp. $u'_i$ or $v'_i$) the duplicator chooses $u'_j$ or $v'_j$ (resp. $u'_j$ or $v'_j$) (If the world chosen by spoiler has not be chosen in the previous round the duplicator pick a fresh index $j$. Else the duplicator pick the $j$ corresponding to the world chosen in the previous rounds). Since there are only $n$ distinct values that $i$ can take \todo{choose? t'as dit choose avant!}, the duplicator will win the game in $n$ rounds with her strategy. Thus the property "for every vertices of a graph, $\modalitynumber p \geq \modalitynumber q$ is not expressible in FO. Therefore $\logicKsharp$ is more expressive than FO.

%
%
%
%
%
\end{document}

%% file: macros.tex
\renewcommand{\phi}{\varphi}
\newcommand{\set}[1]{\{#1\}}

\newcommand\setvertices V
\newcommand\setedges E
\newcommand\labeling \ell
\newcommand\setR{\mathbb R}

\newcommand{\myproblem}[3]{\begin{center}
\begin{minipage}{10cm}
\textbf{#1} \\[-7mm]
\begin{itemize}
\item input: #2
\item output: #3
\end{itemize}
\end{minipage}
\end{center}}

\newcommand{\aGNN}{A}
\newcommand{\semanticsof}[1]{[[#1]]}

\newcommand{\statetv}[2]{x_{#1}(#2)}
\newcommand{\statet}[1]{x_{#1}}
\newcommand{\multiset}[1]{\{\{#1\}\}}

\newcommand{\logicKsharp}{K^{\#}}

\newcommand{\Ap}{Ap}
\newcommand{\modalitynumber}{\#}
\newcommand{\istrue}[1]{1_{#1}}

\newcommand\union\cup
\newcommand{\card}[1]{card(#1)}

\newcommand\sigmabold{\vec{\sigma}}